\documentclass[letterpaper, 10pt, conference]{ieeeconf}      

\IEEEoverridecommandlockouts                              

\usepackage{fancyhdr}

\usepackage{xcolor}
\usepackage{comment}

\usepackage{cite}

\usepackage[mathcal]{euscript}

\usepackage[cmex10]{amsmath}
\usepackage{amssymb}

\usepackage{amsthm} 
  
\usepackage{dsfont} 

\usepackage{epsfig} 

\usepackage{multirow}
\usepackage{enumerate}

\usepackage[ruled, vlined, linesnumbered]{algorithm2e}

\newtheoremstyle{plain}
	  {}
	  {}
	  {\itshape}
	  {}
	  {\bfseries}
	  {}
	  {5pt plus 1pt minus 1pt}
	  {}

\newtheoremstyle{definition}
  	  {}
	  {}
	  {\normalfont}
	  {}
	  {\bfseries}
	  {}
	  {5pt plus 1pt minus 1pt}
	  {}
  
\theoremstyle{plain}

\newtheorem{lemma}{Lemma}
\newtheorem{proposition}{Proposition}

\theoremstyle{definition}

\newcommand{\refeq}[1]			{(\ref{#1})} 
\newcommand{\reffig}[1]			{Fig. \ref{#1}} 
\newcommand{\refsec}[1]			{Section \ref{#1}}
\newcommand{\refapp}[1]			{Appendix \ref{#1}}

\newcommand{\refprop}[1]		{Proposition \ref{#1}}

\newcommand{\reflem}[1]			{Lemma \ref{#1}}

\newcommand{\reffn}[1] 		    {\textsuperscript{\ref{#1}}}

\theoremstyle{plain}

\newcommand{\R}  	{\mathbb{R}} 
\newcommand{\dimspace} 	{d}
\newcommand{\radius} 	{\rho}
\newcommand{\ctr} 	{\vect{c}}
\newcommand{\conv}      {\mathrm{conv}} 
\newcommand{\tri}       {\mathrm{T}} 

\newcommand{\ovect}[1] {\begin{bmatrix}\cos #1\\ \sin #1 \end{bmatrix}}
\newcommand{\ovectsmall}[1] {\scalebox{0.85}{$\ovect{#1}$}}
\newcommand{\ovecT}[1] {\tr{\ovect{#1}\!}}
\newcommand{\ovecTsmall}[1] {\scalebox{0.85}{$\ovecT{#1}$}}
\newcommand{\nvect}[1] {\begin{bmatrix}-\sin #1\\ \cos #1 \end{bmatrix}}

\newcommand{\nvecT}[1] {\tr{\nvect{#1}\!}}
\newcommand{\nvecTsmall}[1] {\scalebox{0.85}{$\nvecT{#1}$}}

\newcommand{\xvectsmall}[1] { \tfrac{ \goal - #1 }{ \norm{ \goal - #1 } } }

\newcommand{\hdist}      {\varepsilon}
\newcommand{\hpos}       {\vect{x}_\varepsilon}
\newcommand{\hvel}       {\dot{\vect{x}}_\varepsilon}
\newcommand{\hproj}      {\widecheck{\pos}}
\newcommand{\hext}       {\widehat{\pos}}

\newcommand{\thead}		 {\vect{t}_{\varepsilon}} 
\newcommand{\nhead}      {\vect{n}_{\varepsilon}} 
\newcommand{\rotmat}     {\mat{R}} 
\newcommand{\hhat}       {\widehat{\pos}_\varepsilon}

\newcommand{\motionset}{\mathcal{M}} 

\newcommand{\freespace}	{\mathcal{F}} 
\newcommand{\workspace}	{\mathcal{W}} 
\newcommand{\obstspace}	{\mathcal{O}} 
\newcommand{\ball}      {\mathrm{B}} 

\newcommand{\path}      {\mathrm{p}}
\newcommand{\pathparam} {s}
\newcommand{\smin}      {\pathparam_{\min}}
\newcommand{\smax}      {\pathparam_{\max}}
\newcommand{\startpos}  {\pos_{\mathrm{start}}}
\newcommand{\goalpos}   {\pos_{\mathrm{goal}}}

\newcommand{\pos} 		{\vect{x}} 			
\newcommand{\ort}	    {\theta}			

\newcommand{\linvel}     {v}  
\newcommand{\angvel}     {\omega} 

\newcommand{\focgain}   	{\kappa_{r}} 			
\newcommand{\hgain}   	{\kappa_{\hdist}} 			
\newcommand{\rgain}      {\kappa_{r}} 

\newcommand{\goal}		{\vect{x}^*} 

\newcommand{\gain}   	{\kappa} 			
\newcommand{\ctrl}      {\vect{u}} 			

\newcommand{\safedist}  {\mathrm{dist}}
\newcommand{\safelevel}{\sigma} 

\let\originalleft\left
\let\originalright\right
\renewcommand{\left}{\mathopen{}\mathclose\bgroup\originalleft}
\renewcommand{\right}{\aftergroup\egroup\originalright}
	
\newcommand{\plist}[1] 	{\left(#1\right)} 
\newcommand{\blist}[1]	{\left[ #1 \right]} 
\newcommand{\clist}[1]	{\left\{#1\right\}} 

\newcommand{\vect}[1]   {\mathrm{#1}}
\newcommand{\mat}[1]    {\mathbf{#1}}
\newcommand{\tr}[1] {{#1}^{\mathrm{T}}} 
\newcommand{\norm}[1]  {\|#1\|}
\newcommand{\absval}[1]{\left|#1 \right|} 

\newcommand{\ldf}   {:=} 

\newcommand{\diff} {\mathrm{d}}

\makeatletter
\DeclareRobustCommand\widecheck[1]{{\mathpalette\@widecheck{#1}}}
\def\@widecheck#1#2{%
    \setbox\z@\hbox{\m@th$#1#2$}%
    \setbox\tw@\hbox{\m@th$#1%
       \widehat{%
          \vrule\@width\z@\@height\ht\z@
          \vrule\@height\z@\@width\wd\z@}$}%
    \dp\tw@-\ht\z@
    \@tempdima\ht\z@ \advance\@tempdima2\ht\tw@ \divide\@tempdima\thr@@
    \setbox\tw@\hbox{%
       \raise\@tempdima\hbox{\scalebox{1}[-1]{\lower\@tempdima\box
\tw@}}}%
    {\ooalign{\box\tw@ \cr \box\z@}}}
\makeatother

\title{\LARGE \bf
Adaptive Headway Motion Control and Motion Prediction \\ for Safe Unicycle Motion Design%
}

\author{Aykut \.{I}\c{s}leyen and Nathan van de Wouw and \"{O}m\"{u}r Arslan
\thanks{The authors are with the Department of Mechanical Engineering, Eindhoven University of Technology, P.O. Box 513, 5600 MB Eindhoven, The Netherlands. The authors are also affiliated with the Eindhoven AI Systems Institute. Emails:  \{a.isleyen, n.v.d.wouw, o.arslan\}@tue.nl}%
}

\begin{document}

\maketitle
\thispagestyle{empty}
\pagestyle{empty}

\begin{abstract}
Differential drive robots that can be modeled as a kinematic unicycle are a standard mobile base platform for many service and logistics robots. 
Safe and smooth autonomous motion around obstacles is a crucial skill for unicycle robots to perform diverse tasks in complex environments.
A classical control approach for unicycle control is feedback linearization using a headway point at a fixed headway distance in front of the unicycle.
The unicycle headway control brings the headway point to a desired goal location by embedding a linear headway reference dynamics, which often results in an undesired offset for the actual unicycle position.
In this paper, we introduce a new unicycle headway control approach with an adaptive headway distance that overcomes this limitation, i.e., when the headway point reaches the goal the unicycle position is also at the goal.
By systematically analyzing the closed-loop unicycle motion under the adaptive headway controller, we design analytical feedback motion prediction methods that bound the closed-loop unicycle position trajectory and so can be effectively used for safety assessment and safe unicycle motion design around obstacles.
We present an application of adaptive headway motion control and motion prediction for safe unicycle path following around obstacles in numerical simulations. 
\end{abstract}

\section{Introduction}
\label{sec.Introduction}

Autonomous mobile robots offer flexible automation solutions for many real-life challenges, from assisting people with daily activities (e.g., service robots \cite{kim_etal_RAM2009}) to enhancing transportation systems (e.g., warehouse robots \cite{renan_nascimento_RAS2021}).
A standard choice of a mobile robot base for many such indoor application settings is differential drive robots that can be modeled as a kinematic unicycle \cite{pentzer_brennan_reichard_IROS2014}.
Safe and smooth control of unicycle robots is essential to autonomously and reliably complete different tasks around obstacles \cite{gul_rahiman_alhady_sahal_CE2019}.
Accurate robot motion prediction is a key enabler for safe unicycle motion design in complex environments \cite{philippsen_siegwart_ICRA2003, chakravarthy_debasish_TSM1998, arslan_koditschek_ICRA2017, isleyen_arslan_RAL2022, arslan_isleyen_arXiv2023}.

In this paper, we introduce a new unicycle headway controller that uses an adaptive headway distance to asymptotically bring both the headway point and the unicycle position to any given goal position. 
For the safety assessment of the close-loop unicycle motion, we propose analytic (circular and triangular) feedback motion prediction methods to accurately bound the closed-loop unicycle motion trajectory under the adaptive headway controller, as illustrated in \reffig{fig.motion_prediction_demo}. 
We apply the proposed unicycle adaptive headway motion control and motion prediction for safe path following around obstacles.

\begin{figure}[t]
\centering
\begin{tabular}{@{}c @{\hspace{0.02\columnwidth}} c @{}}
\includegraphics[width = 0.475\columnwidth]{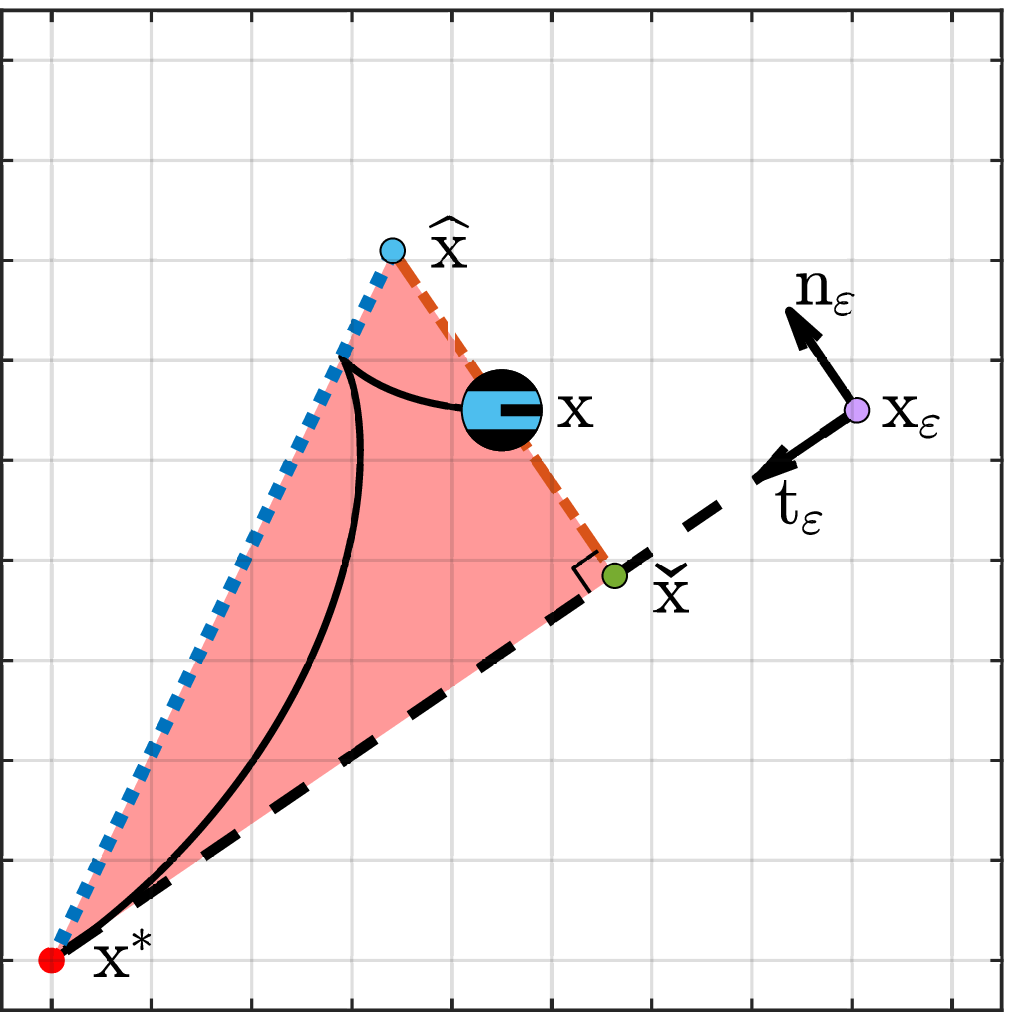} &
\includegraphics[width = 0.475\columnwidth]{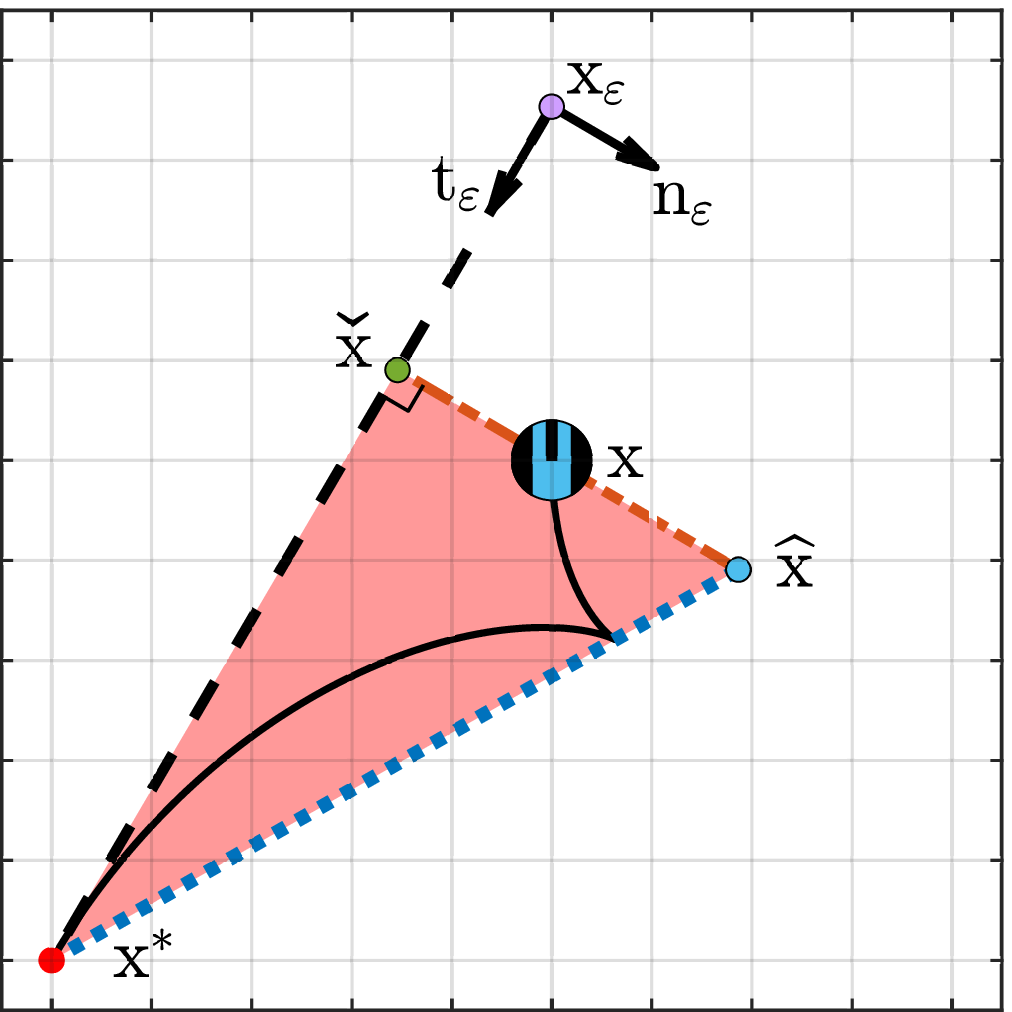} 
\end{tabular}
\caption{Example closed-loop unicycle motion trajectory $\pos(t)$ (solid black line) under the adaptive headway motion control towards a given goal $\goal$ (red point) where the headway point $\hpos$ (purple point) is adaptively placed based on the unicycle position distance to the goal.
The unicycle motion trajectory is bounded by a triangular (red) region defined by the goal position $\goal$, the projected unicycle position $\hproj$ and the extended unicycle position $\hext$. 
}
\label{fig.motion_control_demo}
\vspace{-\baselineskip}
\end{figure}

\subsection{Motivation and Relevant Literature}

Safe autonomous robot motion design requires an accurate understanding and description of the closed-loop robot motion under a feedback motion controller.
Existing control approaches for widely used unicycle mobile robots mainly focus on the stability and convergence of closed-loop unicycle motion \cite{astolfi_JDSMC1999, astolfi_SCL1996,deluca_oriolo_vandittelli_IFAC2000,deluca_oriolo_vandittelli_RAMSETE2002,das_etal_TRA2002}, but pay little attention to the geometric properties of the resulting robot motion, which is essential for safety \cite{isleyen_arslan_RAL2022, isleyen_vandewouw_arslan_arXiv2022}.
A classical feedback linearization approach for unicycle control is based on the use of a headway (a.k.a. offset) point that is at a fixed headway distance in front of the unicycle and is smoothly steered towards a desired goal by embedding some linear headway reference dynamics \cite{das_etal_TRA2002, deluca_oriolo_vandittelli_RAMSETE2002, yun_yamamoto_1992Upenn, petrov_kralov_AMEE2019, novel_campion_bastin_IJRR1995,ren_beard_book2007,gamage_mann_gosine_ACC2010}.
However, the use of a fixed headway distance causes a nonzero steady-state error for the unicycle position since the unicycle robot stops at a headway distance away from the goal while the headway point approaches the goal location.   
In this paper, we propose a unicycle adaptive headway controller based on an adaptive headway distance that asymptotically decreases to zero as the headway point converges to the goal location, which also ensures that the unicycle reaches the goal.

Motion prediction for anticipating the future motion of an autonomous system plays a key role in the safety assessment, control, and planning of mobile robots around obstacles \cite{lefevre_vasquez_laugier_ROBOMECH2014}.
Feedback motion prediction, i.e., finding a bounding motion set on the closed-loop motion trajectory of an autonomous mobile robot under a known control policy, enables informative safety verification and assessment tools \cite{arslan_arXiv2022, isleyen_arslan_RAL2022, arslan_isleyen_arXiv2023,arslan_koditschek_ICRA2017}. 
Reachability analysis offers computational tools for estimating such motion sets for control systems \cite{althoff_dolan_TR02014, althoff_frehse_girard_ARCRAS2021}, but often comes with a high computational cost.
For globally convergent autonomous systems, the notion of forward and backward reachable sets \cite{mitchell_HSCC2007} is trivial because the forward reachable set corresponds to the closed-loop system trajectory due to the autonomous nature of the system whereas the backward reachability set is the entire state space due to the global convergence.
In this paper,  by exploiting the linearity properties of headway control,  we propose new analytic (circular and triangular) motion prediction methods to bound the unicycle robot motion under the adaptive headway controller.
We apply adaptive headway motion control and motion prediction for safe unicycle path following around obstacles and compare the performance of the circular and triangular motion prediction methods with  the forward simulation of the closed-loop unicycle dynamics.

\subsection{Contributions and Organization of the Paper}

This paper introduces new adaptive headway motion control and motion prediction methods for safe unicycle motion design around obstacles.
In \refsec{sec.UnicycleControl}, we present our adaptive headway control approach that asymptotically brings the kinematic unicycle model to a desired goal position by using an adaptive headway distance that depends on the unicycle distance to the goal.
As opposed to a fixed headway distance, our use of an adaptive headway distance allows for both the headway point and the unicycle position to asymptotically reach the goal position. 
In \refsec{sec.UnicyclePrediction}, based on a systematic and careful analysis of the closed-loop unicycle motion, we design two analytical (circular and triangular) feedback motion prediction methods to bound the closed-loop unicycle motion trajectory of the kinematic unicycle model under the adaptive headway control. 
In \refsec{sec.SafeNavigation}, we present an example application of the adaptive headway motion control and motion prediction for safe path following around obstacles in numerical simulations.  
We conclude in \refsec{sec.Conclusions} with a summary of our contributions and future directions.

\section{Unicycle Adaptive Headway Control}
\label{sec.UnicycleControl}

In this section, we briefly describe the standard headway control approach for feedback linearization of the kinematic unicycle robot model, and then present a new unicycle headway control approach with an adaptive headway distance to reach a given goal location.
We highlight important geometric properties of the proposed unicycle adaptive headway controller to understand the resulting unicycle robot motion.

\subsection{Kinematic Unicycle Robot Model}

In the Euclidean plane $ \R^2$, we consider a kinematic unicycle robot whose state is represented by its position $\pos \in \R^2$ and forward orientation angle $\ort \in [ -\pi, \pi )$ that is measured in radians counterclockwise from the horizontal axis.
The equations of motion of the kinematic unicycle robot model are given by
\begin{align} \label{eq.UnicycleDynamics}
\dot{\pos} = \linvel \ovect{\ort} \quad \text{and} \quad \dot{\ort} = \angvel 
\end{align}
where  $\linvel \in \R$ and $\angvel \in \R$ are the scalar control inputs, respectively, specifying the linear and angular velocity of the unicycle robot.
Hence, by definition, the unicycle robot model is underactuated (i.e., three state variables, but only two control inputs) and has the nonholonomic motion constraint of no sideway motion, i.e., $\nvecTsmall{\ort}\dot{\pos} = 0$.

\subsection{Unicycle Headway Motion Control} 

A standard feedback linearization approach for unicycle control \cite{gamage_mann_gosine_ACC2010,koung_etal_ICARCV2020,das_etal_TRA2002} is the use of a headway (a.k.a. offset) point, denoted by $\hpos \in \R^{2}$, that is at a certain (e.g., fixed or varying) headway distance $\hdist \geq 0$ in front of the robot as 
\begin{align} \label{eq.HeadwayPoint}
\hpos \ldf \pos + \hdist \begin{bmatrix} \cos\ort \\ \sin\ort	\end{bmatrix}
\end{align}
so that the nonholonomic unicycle dynamics can be controlled by embedding some desired (e.g., first-order linear) reference dynamics for the headway point.
Under the unicycle dynamics in \refeq{eq.UnicycleDynamics}, the headway point evolves as 
\begin{subequations}\label{eq.HeadwayPointDynamics}
\begin{align}
\hvel &= (\linvel + \dot{\hdist}) \ovect{\ort} + \angvel\,\, \hdist \begin{bmatrix} -\sin\theta \\ \cos\theta \end{bmatrix}
\\
& = \begin{bmatrix} \cos\theta & -\hdist \sin\theta \\ \sin\theta & \hspace{1em}\hdist \cos\theta \end{bmatrix} \begin{bmatrix} \linvel+\dot{\hdist} \\ \angvel \end{bmatrix}.
\end{align}
\end{subequations}
Hence, given a desired reference headway velocity profile $\hvel$  and a desired headway distance function $\hdist$, 
the linear  and angular velocity control inputs for a unicycle robot can be determined 
for $\hdist \neq 0$ as
\begin{align} \label{eq.HeadwayKinematics}
\begin{bmatrix} \linvel \\ \angvel \end{bmatrix} = \begin{bmatrix} \cos\theta & \sin\theta \\ \frac{-\sin\theta}{\hdist} & \frac{\cos\theta}{\hdist} \end{bmatrix} \dot{\pos}_\hdist - \begin{bmatrix}    \dot{\hdist} \\ 0\end{bmatrix}.
\end{align}
For example, a classical choice of reference dynamics for the headway point uses the first-order proportional error feedback to move the headway point $\hpos$ towards a given goal position $\goal \in \R^2$ as \cite{gamage_mann_gosine_ACC2010,koung_etal_ICARCV2020,das_etal_TRA2002}
\begin{align} \label{eq.ReferenceDynamics}
\hvel = - \rgain (\hpos - \goal)
\end{align}
where $\rgain > 0$ is a scalar positive control gain;  
and as a headway distance, the existing literature on unicycle headway motion control \cite{gamage_mann_gosine_ACC2010,koung_etal_ICARCV2020,das_etal_TRA2002} mainly assumes a fixed positive headway distance (i.e., $\hdist > 0$ and  $\dot{\hdist} = 0$), which results in the following standard unicycle headway controller 
\begin{align} \label{eq.UnicycleHeadwayControl}
\begin{bmatrix} \linvel \\ \angvel \end{bmatrix} =  - \rgain \begin{bmatrix} \cos\theta & \sin\theta \\ \frac{-\sin\theta}{\hdist} & \frac{\cos\theta}{\hdist} \end{bmatrix} (\pos - \goal) - \rgain \begin{bmatrix}
\hdist \\ 0
\end{bmatrix}
\end{align} 
that asymptotically brings the headway point to the goal but leaves the robot at a headway distance away from the goal.

\subsection{Unicycle Control with Adaptive Headway Distance}

In order to exactly move the unicycle robot to the goal position using the headway control approach, we consider an adaptive headway distance based on the Euclidean distance of the unicycle position to the goal position as 
\begin{align} \label{eq.HeadwayDist}
\hdist \ldf \hgain \norm{\pos - \goal},
\end{align}
where $\hgain > 0 $ is a fixed scalar coefficient.
Under the unicycle dynamics in \refeq{eq.UnicycleDynamics}, the time rate of change of the headway distance $\hdist$ in \refeq{eq.HeadwayDist} is given by
\begin{align}  \label{eq.HeadwayDistKinematics}
\dot{\hdist} = \hgain \linvel \ovecT{\ort} \frac{(\pos - \goal)}{\norm{\pos - \goal}}
\end{align}
for any $ \pos \neq \goal $, which depends on the linear velocity input $\linvel$.  
Therefore, using the general form of the unicycle headway control in \refeq{eq.HeadwayKinematics}, the first-order headway reference dynamics in \refeq{eq.ReferenceDynamics}, the adaptive headway distance in \refeq{eq.HeadwayDist}, and the headway distance dynamics in \refeq{eq.HeadwayDistKinematics}, we design  an \emph{unicycle adaptive headway motion controller}, denoted by $\ctrl_{\goal}(\pos, \ort) = (\linvel_{\goal}(\pos, \ort), \angvel_{\goal}(\pos, \ort))$, that determines the linear velocity $\linvel_{\goal}(\pos,\ort)$ and the angular velocity $\angvel_{\goal}(\pos,\ort)$ for the unicycle model in \refeq{eq.UnicycleDynamics} as\footnote{\label{fn.fn1}Note that we set $\linvel \!=\! 0$ and $\angvel \!=\! 0$ when the unicycle is at the goal (i.e., $\pos \!=\! \goal$) to resolve the indeterminacy. This naturally introduces a discontinuity in control at the goal position as necessitated by Brockett's theorem \cite{brockett_DGCT1983}. Otherwise, the unicycle adaptive headway motion control in \refeq{eq.AdaptiveHeadwayControl} is locally Lipshitz continuous everywhere, away from the goal position.}
\begin{subequations} \label{eq.AdaptiveHeadwayControl}
\begin{align} 
\linvel_{\goal}(\pos,\ort) &= \frac{\focgain \norm{\goal - \pos} \plist{ \ovecTsmall{\ort} \xvectsmall{\pos} - \hgain }}{ 1- \hgain \ovecTsmall{\ort} \xvectsmall{\pos} } \label{eq.AdaptiveHeadwayControlVelocity}
\\
\angvel_{\goal}(\pos,\ort) &= \frac{\focgain}{\hgain} \nvecTsmall{\ort} \xvectsmall{\pos}
\end{align}
\end{subequations}
where $\focgain > 0$ and $1 > \hgain > 0$.
Here, it is important to remark that the unity upper bound on $\hgain$ is not only a sufficient but also a necessary condition to avoid the singularity in the linear velocity control in \refeq{eq.AdaptiveHeadwayControlVelocity} and also to ensure the global convergence of our unicycle adaptive headway controller (see \refprop{prop.GlobalStability}).

The major significance of the adaptive headway distance in \refeq{eq.HeadwayDist} over a fixed headway distance is that being at the goal is the same for both the headway point and the unicycle position.
\begin{lemma} \label{lem.SimultaneousConvergence}
\emph{(Being at the Goal)}
Regardless of the unicycle orientation $\ort \in [\pi, \pi)$, the unicycle position $\pos $ is at the goal $\goal$ if and only if the headway point $\hpos$ associated with the adaptive headway distance $\hdist$ in \refeq{eq.HeadwayDist} is at the goal $\goal$ , i.e.,
\begin{align} \label{eq.Convergence}
\pos \!=\! \goal \Longleftrightarrow \hpos \!=\! \goal \quad \forall \ort \in [-\pi, \pi).
\end{align}
\end{lemma}
\begin{proof}
See \refapp{app.SimultaneousConvergence}.
\end{proof}

Hence,  as the headway point is asymptotically approaching the goal location, the unicycle robot also reaches the goal under the adaptive headway controller.

\begin{proposition} \label{prop.GlobalStability}
\emph{(Global Convergence)}
The unicycle adaptive headway motion controller $\ctrl_{\goal}$ in \refeq{eq.AdaptiveHeadwayControl} asymptotically brings all initial unicycle states $(\pos, \ort)$ in $\R^{2} \times [-\pi, \pi)$ to any given goal position $\goal \in \R^2$, that is to say, the closed-loop trajectory $(\pos(t), \ort(t))$ of the unicycle dynamics in \refeq{eq.UnicycleDynamics} under the adaptive headway controller in \refeq{eq.AdaptiveHeadwayControl} satisfies
\begin{align}
\lim_{t\rightarrow \infty} \pos(t) = \goal.
\end{align}   
\end{proposition}
\begin{proof}
By construction, the unicycle adaptive headway distance control policy in \refeq{eq.AdaptiveHeadwayControl} realizes the first-order headway-point reference dynamics in \refeq{eq.ReferenceDynamics}. 
The headway point $\hpos$ under the reference dynamics in \refeq{eq.ReferenceDynamics} asymptotically reaches the goal position since the squared Euclidean distance of the headway point $\hpos(\pos,\ort)$ to the goal decreases over time as
\begin{align}
\frac{\diff}{\diff t} \norm{\hpos - \goal}^2 \!=\! - 2\focgain \norm{\hpos - \goal}^2 \leq 0. 
\end{align}
Therefore, we also have the global convergence of the unicycle position to the goal position since being at the goal is the same for both the headway point and the unicycle position (\reflem{lem.SimultaneousConvergence}), i.e., $\pos = \goal \Longleftrightarrow \hpos = \goal$, due to the specific form of the adaptive headway distance in \refeq{eq.HeadwayDist} 
\end{proof}

\subsection{Geometric Properties of Adaptive Headway Controller}
\label{sec.GeometricPropertiesAHC}

In this part, we present some important geometric properties of the unicycle robot motion under the adaptive headway controller that form the basis for the unicycle feedback motion prediction later in \refsec{sec.UnicyclePrediction}.
Since the headway point $\hpos$ moves along a straight line segment towards the goal $\goal$ under the headway reference dynamics in \refeq{eq.ReferenceDynamics}, it is convenient to define the tangent vector $\thead$ and the normal vector $\nhead$ of the motion of the headway point as 
\begin{subequations}\label{eq.HeadwayTangentNormal}
\begin{align}
\thead &\ldf \left \{ \begin{array}{cl}
\dfrac{\goal - \hpos}{\norm{\goal - \hpos}} & \text{, if } \hpos \neq \goal 
\\
0 & \text{, otherwise}
\end{array}
\right.  
\label{eq.HeadwayTangent}
\\
\nhead &\ldf \left \{ 
\begin{array}{cl}
\rotmat_{+\frac{\pi}{2}} \thead & \text{, if } \tr{(\goal - \pos)} \nvect{\theta} \geq 0 
\\
\rotmat_{-\frac{\pi}{2}} \thead & \text{, otherwise } 
\end{array}
\right.
\label{eq.HeadwayNormal}
\end{align} 
\end{subequations}
where $ \rotmat_{\phi} :=  \begin{bmatrix}  \cos \phi & -\sin \phi \\ \sin \phi & \cos \phi \end{bmatrix}$ denotes the 2D rotation matrix with an angle of $\phi$.
Observe that both the tangent  $\thead$ and the normal $\nhead$ are constant during the unicycle motion under the adaptive headway controller away from the goal.
We also define the projected robot position $\hproj$ and the extended robot position $\hext$ with respect to the motion of the headway point as 
\begin{subequations}\label{eq.ProjectedExtendedPosition}
\begin{align}
\hproj & \ldf \goal + \thead \tr{\thead} (\pos - \goal) 
\label{eq.HeadwayProjectedDefinition}
\\
\hext & \ldf \hproj + \frac{\hgain}{\sqrt{1 - \hgain^2}} \norm{\hproj - \goal} \nhead 
\label{eq.HeadwayExtendedDefinition}
\end{align}
\end{subequations}
where the distances of the projected and extended robot positions to the goal satisfy\reffn{fn.ProjectedExtendedDistance}
\begin{subequations}
\begin{align}
\norm{\hproj - \goal} &= \tr{\thead} (\goal - \pos) 
\\
\norm{\hext - \goal} &= \frac{1}{\sqrt{1 - \hgain^2}} \norm{\hproj - \goal}.
\end{align}
\end{subequations}

\addtocounter{footnote}{1}
\footnotetext{\label{fn.ProjectedExtendedDistance}For any $0 < \hgain < 1$ and  $\pos \neq \goal$, the following relations hold 
\begin{align*}
\tr{\thead}(\goal - \pos) &= \frac{\norm{\goal - \pos}^2}{\norm{\goal - \hpos}}\plist{1 - \hgain \ovecTsmall{\ort} \frac{(\goal - \pos)}{\norm{\goal - \pos}}} \geq 0
\\
\norm{\goal - \hproj} &= \absval{\tr{\thead}(\goal - \pos)} =  \tr{\thead}(\goal - \pos).
\end{align*}
}

A critical property of the projected and extended unicycle positions is that they bound the actual unicycle position.
\begin{lemma} \label{lem.RobotbtwProjectedandExtended}
\emph{(Unicycle Position Bound)}
For any unicycle state $(\pos, \ort) \in \R^2 \times [-\pi, \pi)$,  the unicycle position $\pos$ is in between the projected unicycle position $\hproj$ and the extended unicycle position $\hext$, i.e.,
\begin{align}
    \pos \in \blist{\hproj, \hext},
\end{align}
where $\blist{\vect{a}, \vect{b}}:= \clist{\alpha \vect{a} + (1 - \alpha) \vect{b}  \, \big | \, \alpha \in [0, 1]}$ denotes the straight line segment between points $\vect{a}$  and  $\vect{b}$.
\end{lemma}
\begin{proof}
See \refapp{app.RobotbtwProjectedandExtended}.
\end{proof}

\begin{lemma} \label{lem.DistanceMetricsOrder}
\emph{(Unicycle Distance-to-Goal Bound)}
For any unicycle state $(\pos, \ort) \in \R^2 \times [-\pi, \pi)$, the Euclidean distance $\norm{\pos-\goal}$ of the unicycle position to the goal position $\goal$ is bounded below and above by the distances of the projected and extended unicycle positions, $\hproj$ and $\hext$, to the goal as
\begin{align}
\norm{\hproj - \goal} \leq \norm{\pos - \goal} \leq \norm{\hext - \goal}. 
\end{align}
\end{lemma}
\begin{proof}
See \refapp{app.DistanceMetricsOrder}.
\end{proof}

Due to their strong geometric relation with the unicycle position in  \reflem{lem.RobotbtwProjectedandExtended} and \reflem{lem.DistanceMetricsOrder}, it is important to understand how the projected and extended unicycle positions change under the adaptive headway controller in order to understand the closed-loop unicycle motion. 

\begin{lemma}\label{lem.ProjectedExtendedDynamics}
\emph{(Motion of Projected/Extended Unicycle Positions)}
For any unicycle state $(\pos, \ort) \in \R^2 \!\times\! [-\pi, \pi)$, the projected unicycle position $\hproj$ and the extended unicycle position $\hext$ evolve under the unicycle adaptive headway controller in \refeq{eq.AdaptiveHeadwayControl} towards any given goal position $\goal \in \R^{2}$ as 
\begin{align}
\dot{\hproj} =  - \kappa \plist{\hproj - \goal} 
\quad \text{and} \quad
\dot{\hext} =  - \kappa \plist{\hext - \goal}
\end{align}
where $\kappa = \frac{\norm{\pos - \goal}^2}{\norm{\hproj - \goal}} \plist{\ovecTsmall{\ort} \thead}^2$  if $\pos \neq \goal$, (and zero otherwise).
Hence, the distances of the projected and extended unicycle positions to the goal are nonincreasing, i.e.,
\begin{align}
&\frac{\diff }{\diff t} \norm{\hproj-\goal}^2  \leq 0, \text{and }\frac{\diff }{\diff t}  \norm{\hext-\goal}^2 \leq 0
\end{align}
and their solution trajectories satisfy for all $t \geq 0$ that
\begin{align}
&\hproj(t) \in \blist{\goal, \hproj(0)} 
\quad \text{and} \quad
\hext(t) \in \blist{\goal, \hext(0)}. 
\end{align}
\end{lemma}
\begin{proof}
See \refapp{app.ProjectedExtendedDynamics}.
\end{proof}

Finally, as summarized below, two important geometric features of the unicycle adaptive headway controller related to the unicycle orientation are continuous goal alignment and goal-aligned forward unicycle motion.

\begin{lemma} \label{lem.GoalAlignment}
\emph{(Goal Alignment)}
At any unicycle state $(\pos, \ort) \in \R^2 \times [-\pi, \pi)$ away from the goal position $\goal \!\in\! \R^{2}$ (i.e., $\pos \neq \goal$), the unicycle adaptive headway controller in \refeq{eq.AdaptiveHeadwayControl} adjusts the unicycle orientation towards the goal $\goal \!\in\! \R^{2}$ as 
\begin{align*}
&\frac{\diff}{\diff t}\plist{\ovecTsmall{\ort} \xvectsmall{\pos} } \nonumber 
\\
& \hspace{3mm} \geq  \frac{\rgain}{\hgain} \plist{\!\nvecTsmall{\ort}\! \xvectsmall{\pos}\!}^{\!\!2} \!\plist{\!1 \!-\! \hgain \ovecTsmall{\ort}\! \xvectsmall{\pos} \!\!}  \geq  0
\end{align*}
which is strictly positive when  $\nvecTsmall{\ort}\! \xvectsmall{\pos} \neq 0$.
\end{lemma}
\begin{proof}
See \refapp{app.GoalAlignment}.
\end{proof}

\begin{lemma} \label{lem.EuclideanDistance2Goal}
\emph{(Goal-Aligned Forward Motion)} 
Starting at $t=0$ from any initial unicycle state $(\pos_0, \ort_0) \!\in \R^2 \!\times\! [-\pi, \pi)$ that is aligned with the goal  $\goal \!\in\! \R^2$ as $\ovecTsmall{\ort_0}\! \xvectsmall{\pos_0} \!>\! \hgain$, the unicycle distance to the goal $\norm{\pos(t) - \goal}$ along the solution trajectory $(\pos(t), \ort(t))$ of the unicycle dynamics in \refeq{eq.UnicycleDynamics} under the adaptive headway controller in \refeq{eq.AdaptiveHeadwayControl} is decreasing over time and the unicycle moves in the forward direction with a positive velocity for all future times $t \geq 0$, i.e.,
\begin{align}
\frac{\diff}{\diff t} \norm{\pos(t) - \goal}^2 \leq 0   \quad \text{and} \quad 
 \linvel_{\goal}(\pos(t), \ort(t)) \geq 0 
\end{align}
where the inequalities are strict for $\pos(t) \neq \goal$.
\end{lemma}
\begin{proof}
See \refapp{app.EuclideanDistance2Goal}.
\end{proof}

\section{Unicycle Feedback Motion Prediction \\ for Adaptive Headway Control}
\label{sec.UnicyclePrediction}

In this section, we present two (one circular and one triangular) feedback motion prediction methods, as illustrated in \reffig{fig.motion_prediction_demo}, for bounding the closed-loop motion trajectory of the unicycle robot model under the adaptive headway controller and show that these motion prediction methods asymptotically shrink to the goal point and has a Lipschitz-continuous minimum (e.g., collision) distance to any given (e.g., obstacle) point, which are essential for provably correct and safe robot motion design \cite{isleyen_arslan_RAL2022, arslan_arXiv2022}.

\subsection{Circular Unicycle Feedback Motion Prediction}

One can use the decaying distance of the (extended) unicycle position to the goal (\reflem{lem.ProjectedExtendedDynamics} and \reflem{lem.EuclideanDistance2Goal}) in order to determine the closed-loop unicycle motion range. 

\begin{proposition} \label{prop.CircularUnicycleMotionPrediction}
\emph{(Circular Unicycle Motion Prediction)} Starting at $t = 0$ from any initial state $(\pos_0, \ort_0) \in \R^2 \times [-\pi, \pi)$, the unicycle position trajectory $\pos(t)$ under the adaptive headway controller $\ctrl_{\goal}$ in \refeq{eq.AdaptiveHeadwayControl} towards any given goal $\goal \in \R^{2}$ is contained for all future times $t\geq 0$ in the circular motion prediction set $\motionset_{\ctrl_{\goal}, \ball}(\pos_0, \ort_0)$ that is defined as 
\begin{align}\label{eq.CircularMotionPrediction}
\motionset_{\ctrl_{\goal}, \ball}(\pos_0, \ort_0) \!\ldf\!  \left \{ 
\begin{array}{@{}l@{}l@{}}
\ball\plist{\goal\!, \norm{\pos_0 \!- \!\goal}} & \text{, if } \ovecTsmall{\ort_0}\!\!\! \xvectsmall{\pos_0} \!\geq\! \hgain\\
\ball\plist{\goal\!, \norm{\hext_0 \!-\! \goal}}  & \text{, otherwise}
\end{array}
\right.
\end{align}
where  $\ball(\ctr,\radius) \!\ldf\! \clist{ \vect{z} \!\in\! \R^{2} \big|  \norm{\vect{z} \!-\! \ctr } \leq \radius} $ is the Euclidean closed ball centered at $\vect{c} \!\in\! \R^{2}$ with radius $\radius \!\geq\! 0$, and $\hext$ is the extended unicycle position associated with unicycle state $(\pos, \ort)$  as defined in \refeq{eq.HeadwayExtendedDefinition}.   
\end{proposition}
\begin{proof}
If $\ovecTsmall{\ort_0}\! \xvectsmall{\pos_0} \!\geq\! \hgain$, then the unicycle moves forward and its distance to the goal $\norm{\pos(t) - \goal}$ is decreasing along the motion trajectory (\reflem{lem.EuclideanDistance2Goal}).
Otherwise, we have that $\norm{\pos(t) - \goal}$ is bounded above by $\norm{\hext(t) - \goal}$ (\reflem{lem.DistanceMetricsOrder}) which always decreases under the adaptive headway controller (\reflem{lem.ProjectedExtendedDynamics}). 
Thus, the result follows.   
\end{proof} 

\begin{figure}[t]
\centering
\begin{tabular}{@{}c @{\hspace{0.02\columnwidth}} c @{}}
\includegraphics[width = 0.485\columnwidth]{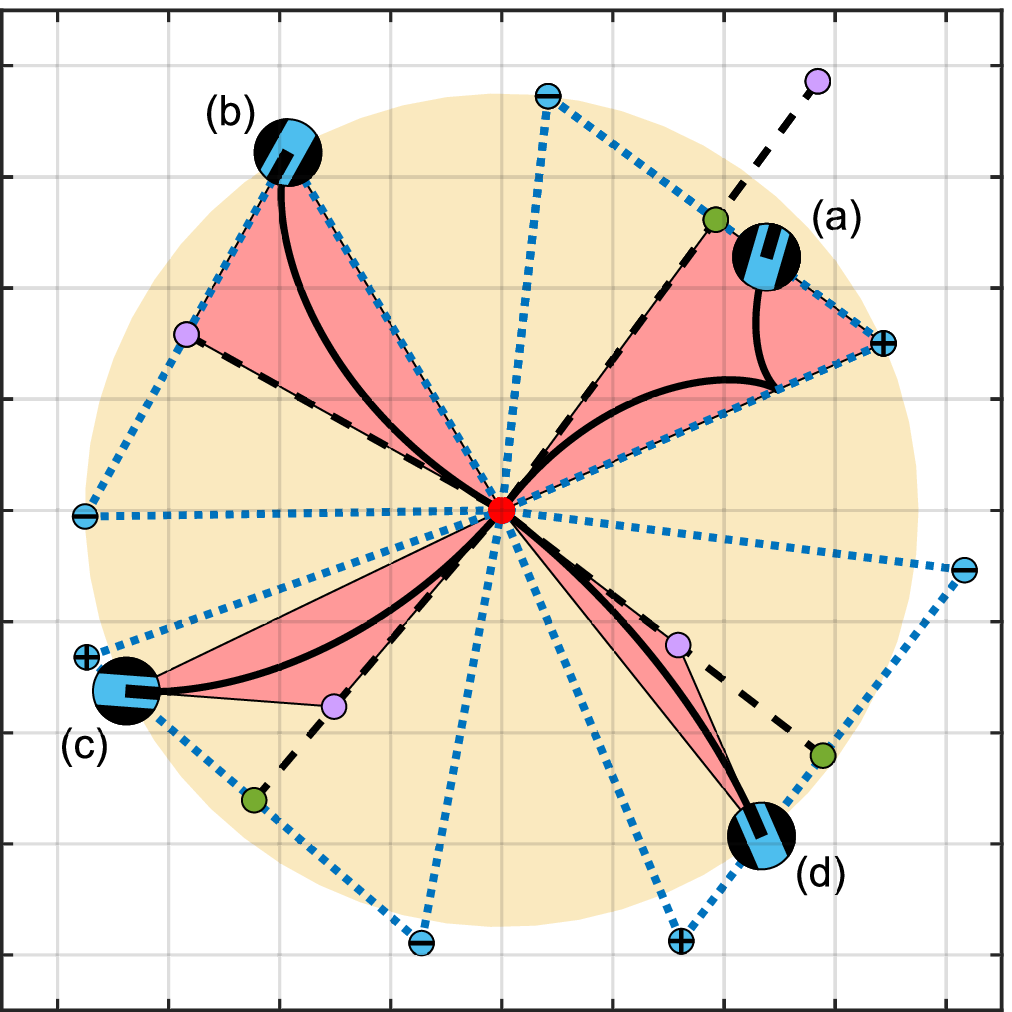} &
\includegraphics[width = 0.485\columnwidth]{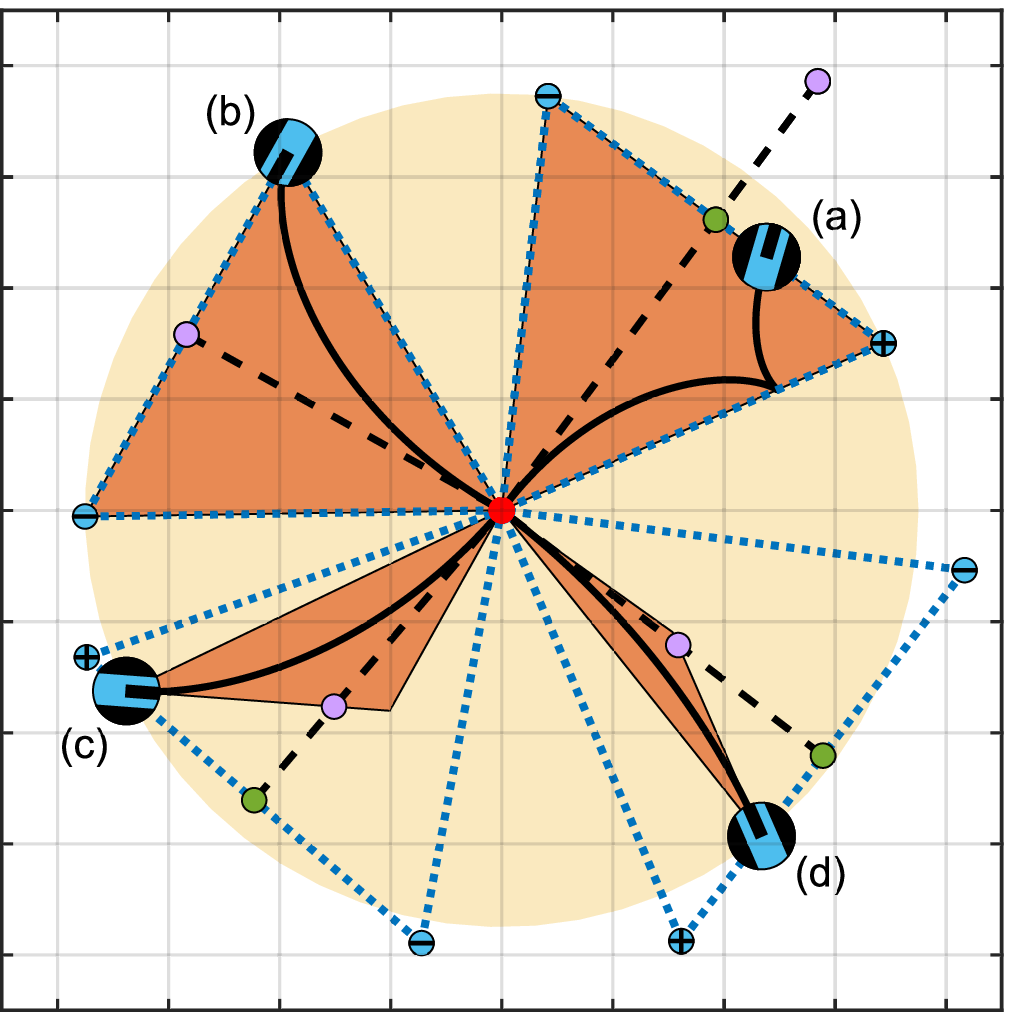} 
\end{tabular}
\caption{Triangular motion bound (left, red) and triangular motion prediction (right, orange) that contain the closed-loop unicycle motion trajectory (solid black line) of the adaptive headway controller towards a given goal (red point), starting from different initial unicycle states that share the same circular motion prediction (yellow).
The triangular motion prediction (orange) is an extension of the triangular motion bound (red) in order to ensure a Lipschitz-continuous distance-to-collision measure.  
(a) Negative initial linear velocity (i.e., backward motion), (b) Zero initial linear velocity (i.e., transition from backward to forward motion), (c, d) Positive initial linear velocity (i.e., forward motion).
}
\label{fig.motion_prediction_demo}
\vspace{-\baselineskip}
\end{figure}

An important property of circular unicycle motion prediction is positive inclusiveness, which ensures that a safety assessment based on the distance of feedback motion prediction set to obstacles is consistent for all future times.
\begin{proposition} \label{prop.PositiveInclusionCircularMotionPrediction}
\emph{(Positive Inclusion of Circular Motion Prediction)}
The circular motion prediction set $\motionset_{\ctrl_{\goal}, \ball}(\pos,\ort)$ of the adaptive headway controller $\ctrl_{\goal}$ towards any given goal position $\goal \in \R^2$ is positively inclusive along the resulting unicycle motion trajectory $(\pos(t), \ort(t))$, i.e.,
\begin{align}
\motionset_{\ctrl_{\goal}, \ball}(\pos(t), \ort(t)) \supseteq \motionset_{\ctrl_{\goal}, \ball}(\pos(t'), \ort(t'))\quad \forall t' \geq t. 
\end{align} 
\end{proposition}
\begin{proof}
The results follows from the fact that $\norm{\pos(t) - \goal} \leq \norm{\hext(t) - \goal}$ (\reflem{lem.DistanceMetricsOrder}), and  $\norm{\hext(t) - \goal}$ is decreasing along the unicycle motion trajectory (\reflem{lem.ProjectedExtendedDynamics}), and  $\norm{\pos(t) - \goal}$ start persistently decreasing once $\ovecTsmall{\ort(t)}\! \xvectsmall{\pos(t)} \!\geq\! \hgain$ (\reflem{lem.EuclideanDistance2Goal}). 
\end{proof}

\begin{proposition} \label{prop.CircularMotionPredictionRadius}
\emph{(Circular Motion Prediction Radius)}
The circular motion prediction set $\motionset_{\ctrl_{\goal}, \ball}(\pos(t),\ort(t))$ asymptotically shrinks to the goal position $\goal$ along the closed-loop motion trajectory $(\pos(t), \ort(t))$ of the unicycle adaptive headway controller as its radius  
asymptotically decays to zero, i.e.,
\begin{align}
\lim_{t \rightarrow \infty} \min_{\pos'\in \motionset_{\ctrl_{\goal}, \ball}(\pos(t),\ort(t))} \norm{\pos'- \goal} = 0.  
\end{align} 
\end{proposition}
\begin{proof}
The result follows from the fact that both the robot position $\pos(t)$ and the extended robot position $\hext(t)$ asymptotically converge to the goal $\goal$ (\refprop{prop.GlobalStability} and \reflem{lem.ProjectedExtendedDynamics}), which define the radius of the circular motion prediction set in \refeq{eq.CircularMotionPrediction}.
\end{proof}

\begin{proposition}\label{prop.CircularMotionPredictionDistance}
\emph{(Circular Motion Prediction Distance)}
For any unicycle state $(\pos, \ort) \in \R^{2}\times [\pi, -\pi)$ away from the goal $\goal$, the minimum distance $\min_{\pos'\in \motionset_{\ctrl_{\goal}, \ball}(\pos,\ort)} \norm{\pos'- \vect{z}}$ of the circular motion prediction set $\motionset_{\ctrl_{\goal}, \ball}(\pos,\ort)$ of the adaptive headway controller $\ctrl_{\goal}$ to any given point $\vect{z} \in \R^2$
 is a locally Lipschitz continuous function of  unicycle position $\pos$, unicycle orientation $\ort$ and goal position $\goal$.
\end{proposition}
\begin{proof}
Due to the circular shape of $\motionset_{\ctrl_{\goal}, \ball}(\pos,\ort)$,  its minimum distance to a point is determined by its center distance $\norm{\goal - \vect{z}}$ and its radius, which is $\norm{\pos - \goal}$ if   $\ovecTsmall{\ort_0}\!\!\! \xvectsmall{\pos_0} \!\geq\! \hgain$; and $\norm{\hext - \goal}$ otherwise. 
Note that for each case, the radius of $\motionset_{\ctrl_{\goal}, \ball}(\pos,\ort)$ is locally Lipschitz continuous with respect to $\pos$, $\ort$, and $\goal$.
Now observe that if $\ovecTsmall{\ort}\!\!\! \xvectsmall{\pos} = \hgain$, then $\hpos = \hproj$ and $\norm{\hproj - \goal} = \norm{\hpos - \goal} = \sqrt{1 - \hgain^2} \norm{\pos - \goal}$. 
Hence, $\norm{\hext - \goal} = \frac{1}{\sqrt{1 - \hgain^2}} \norm{\hproj - \goal} = \norm{\pos - \goal}$.
Hence, the radius of $\motionset_{\ctrl_{\goal}, \ball}(\pos,\ort)$ is a continuous selection of locally Lipschitz continuous functions which is also locally Lipschitz \cite{liu_JCO1995}. 
Thus, the result follows.
\end{proof}

\subsection{Triangular Motion Range Prediction}

Although it has a simple analytical form, the circular unicycle motion prediction $\motionset_{\ctrl_{\goal}, \mathrm{B}}(\pos, \ort)$ is conservative in describing the closed-loop unicycle motion due to its symmetric form as illustrated in \reffig{fig.motion_prediction_demo}.
To capture unicycle motion direction more accurately, we introduce a triangular unicycle motion prediction that contains the closed-loop unicycle motion trajectory under adaptive headway control.

\begin{lemma} \label{prop.TriangularUnicycleMotionBound}
\emph{(Triangular Unicycle Motion Bound)}
Starting at $t = 0$ from any initial unicycle pose $(\pos(0), \ort(0)) \in \R^2 \times [-\pi, \pi)$, the unicycle position trajectory $\pos(t)$ under the adaptive headway control $\ctrl_{\goal}(\linvel, \angvel)$ in \refeq{eq.AdaptiveHeadwayControl} towards a given goal $\goal \in \R^{2}$ is contained for all future times in a triangular set as 
\begin{align*}
\pos(t) \in
\left \{ 
\begin{array}{@{}l@{\,}l@{}}
\conv (\goal,\! \pos(0), \hpos(0) ) & \text{, if } \ovecTsmall{\ort(0)} \!\! \xvectsmall{\pos(0)} \geq \hgain\\
\conv (\goal, \hproj(0), \hext(0)) & \text{, otherwise}
\end{array}
\right.
\end{align*}
where one has $\conv (\goal,\! \pos(0), \hpos(0)) \subseteq \conv (\goal, \hproj(0), \hext(0))$ if $\ovecTsmall{\ort(0)} \!\! \xvectsmall{\pos(0)} \geq \hgain$; and also $\hext = \pos$ and $\hproj = \hpos$ if  $\ovecTsmall{\ort} \!\! \xvectsmall{\pos} = \hgain$.
Here, $\conv$ denotes the convex hull operator and $\hpos(0)$ is the headway point defined in \refeq{eq.HeadwayPoint},  $\hproj(0)$ and  $\hext(0)$ are the initial projected and extended unicycle positions defined in \refeq{eq.HeadwayProjectedDefinition} and \refeq{eq.HeadwayExtendedDefinition}, respectively.
\end{lemma}
\begin{proof}
We have from  \reflem{lem.RobotbtwProjectedandExtended} and  \reflem{lem.ProjectedExtendedDynamics} that $\pos(t) \in [\hproj(t), \hext(t)]$ where $\hproj(t) \in [\goal, \hproj(0)]$ and $\hext(t) \in [\goal, \hext(0)]$  for all $t \geq 0$. 
Hence, using the convex combination of the boundary points, one can bound the unicycle position trajectory as $\pos(t) \in \conv(\goal, \hproj(0), \hext(0))$ for all $t \geq 0$. 

If $\ovecTsmall{\ort(0)} \xvectsmall{\pos(0)} \geq \hgain$, then it holds for all future times, i.e. $\ovecTsmall{\ort(t)} \xvectsmall{\pos(t)} \geq \hgain$ (\reflem{lem.GoalAlignment}) and the unicycle robot always moves with a nonnegative linear velocity towards the headway point (\reflem{lem.EuclideanDistance2Goal}). 
Note that the headway point satisfies $\hpos(t) \in [\goal, \hpos(0)]$ due to the reference headway dynamics in \refeq{eq.ReferenceDynamics}.
Hence, the unicycle position trajectory can be bounded as $\pos(t) \in \conv(\goal, \pos(0), \hpos(0))$ since the unicycle velocity $\dot{\pos}(t)$ always points towards the headway point $\hpos(t) \in \conv(\goal, \pos(0), \hpos(0))$  which leaves  $\conv(\goal, \pos(0), \hpos(0))$ positively invariant due to the sub-tangentiality property on the set boundary \cite{blanchini_Automatica1999}.  

Finally, by definitions \refeq{eq.HeadwayPoint}, \refeq{eq.HeadwayProjectedDefinition}, \refeq{eq.HeadwayExtendedDefinition}, we have $\pos \in [\hproj, \hext]$ and $\hpos \in [\goal, \hproj]$ if $\ovecTsmall{\ort(0)} \xvectsmall{\pos(0)} \geq \hgain$; and $\hext = \pos$ and $\hproj = \hpos$ if  $\ovecTsmall{\ort} \!\! \xvectsmall{\pos} = \hgain$. 
Therefore, $\conv(\goal, \pos, \hpos) \subseteq \conv(\goal, \hproj, \hext)$ if $\ovecTsmall{\ort} \xvectsmall{\pos} \geq \hgain$, where the equalities are tight.
\end{proof}

The triangular bound on the unicycle position trajectory in \refprop{prop.TriangularUnicycleMotionBound} changes discontinuously for the goal positions that are placed almost perfectly behind the unicycle( i.e., $\ovecTsmall{\ort}(\goal - \pos) \approx -1$).  
To overcome this discontinuity issue, we construct a triangular motion prediction set, denoted by $\motionset_{\ctrl_{\goal}, \tri}(\pos, \ort)$ for the adaptive headway controller $\ctrl_{\goal}$ as 
\begin{align}\label{eq.TriangularMotionPrediction}
\motionset_{\ctrl_{\goal}, \tri}(\pos, \ort) \!\ldf\! 
\left \{
\begin{array}{@{}l@{}l@{}}
\conv (\goal, \pos, \hhat ) & \text{, if } \ovecTsmall{\ort} \!\!\xvectsmall{\pos} \! \geq\! \hgain
\\
\conv (\goal, \hext_{+}, \hext_{-}) & \text{, otherwise}
\end{array}
\right.
\end{align}
where the triangle vertices are defined for $\pos \neq \goal$, using the headway point $\hpos$ in \refeq{eq.HeadwayPoint}, the headway tangent $\thead$ in \refeq{eq.HeadwayTangent}, and the projected unicycle position $\hproj$ in \refeq{eq.HeadwayProjectedDefinition},  as 
\begin{align}
&\hhat \ldf \hpos + \frac{1 - \ovecTsmall{\ort} \xvectsmall{\pos} }{1 - \hgain}  \hdist \ovectsmall{\ort} \\
&\hext_{+} \ldf \hproj + \frac{\hgain}{\sqrt{1-\hgain}} \norm{\hproj - \goal} \rotmat_{\frac{\pi}{2}} \thead \\
&\hext_{-} \ldf \hproj - \frac{\hgain}{\sqrt{1-\hgain}} \norm{\hproj - \goal} \rotmat_{\frac{\pi}{2}} \thead 
\end{align}
which are all set equal to $\goal$ for $\pos = \goal$. 
Note that the triangular motion prediction set changes continuously because  $\conv (\goal, \pos, \hhat ) \! = \!\conv (\goal, \hext_{+}, \hext_{-})$ when   $\ovecTsmall{\ort}\!\!\!\xvectsmall{\pos} \! =\! \hgain$.
It is also important to observe that $\hpos \! \in\! [\pos, \hhat]$ and $\hproj \!= \!\!\frac{\hext_{+} + \hext_{-}}{2}$, and the extended unicycle position $\hext$ in \refeq{eq.HeadwayExtendedDefinition} is equal to either $\hext_{+}$ or $\hext_{-}$. 
Hence, the triangular motion prediction $\motionset_{\ctrl_{\goal}, \tri}(\pos(0), \ort(0))$  is a superset of the triangular bound on the unicycle position trajectory $\pos(t)$ in \refprop{prop.TriangularUnicycleMotionBound}, i.e., $\pos(t) \in \motionset_{\ctrl_{\goal}, \tri}(\pos(0), \ort(0))$ for all $t\geq 0$.

\begin{proposition}\label{prop.TriangularMotionPredictionRadius}
\emph{(Triangular Motion Prediction Radius)}
The triangular motion prediction set $\motionset_{\ctrl_{\goal}, \tri}(\pos(t), \ort(t))$ of the adaptive headway controller $\ctrl_{\goal}$ in \refeq{eq.AdaptiveHeadwayControl} asymptotically shrinks to the goal point along the resulting unicycle motion trajectory $(\pos(t), \ort(t))$ as its radius with respect to the goal asymptotically decays to zero, i.e.,
\begin{align}
\lim _{t \rightarrow \infty} \min_{\pos ' \in \motionset_{\ctrl_{\goal}, \tri}(\pos(t), \ort(t))} \norm{\pos'- \goal} = 0 
\end{align} 
\end{proposition}
\begin{proof}
The result follows from the fact that the vertices points of the triangular motion prediction $\motionset_{\ctrl_{\goal}, \tri}(\pos(t), \ort(t))$ asymptotically converge to the goal $\goal$. 
\end{proof}  

\begin{proposition}\label{prop.TriangularMotionPredictionDistance}
\emph{(Triangular Motion Prediction Distance)}
For any unicycle state $(\pos, \ort) \in \R^2 \times [-\pi, \pi)$ away from the goal $\goal$, the minimum  distance  $\min_{\pos'\in \motionset_{\ctrl_{\goal}, \tri}(\pos, \ort)} \norm{\pos'- \vect{z}}$ of the triangular motion prediction set $\motionset_{\ctrl_{\goal}, \tri}(\pos, \ort)$ of the adaptive headway controller $\ctrl_{\goal}$ to any given point $\vect{z} \in \R^2$ is a locally Lipschitz continuous function of the  unicycle position $\pos$, the unicycle orientation $\ort$, and the goal position $\goal$.  
\end{proposition}
\begin{proof}
Away from the goal position, the vertex points of the triangular motion prediction set $\motionset_{\ctrl_{\goal}, \tri}(\pos, \ort)$ are smooth functions of $\pos$, $\ort$, and $\goal$ for both the case of $\ovecTsmall{\ort} \!\! \xvectsmall{\pos} \geq \hgain$ and otherwise.
Hence, the triangular feedback motion prediction set $\motionset_{\ctrl_{\goal}, \tri}(\pos, \ort)$ can be expressed as an affine transformation of a fixed triangle that is a smooth function of $\pos$, $\ort$, and $\goal$. Therefore, the distance of $\motionset_{\ctrl_{\goal}, \tri}(\pos, \ort)$ is locally Lipschitz continuous with respect to $\pos$, $\ort$, and $\goal$ since the minimum set-distance is  Lipschitz continuous under smooth affine transformations (see Lemma 1 in \cite{isleyen_arslan_RAL2022}).   
\end{proof}

\section{Application: Safe Unicycle Path Following \\ {\small via Adaptive Headway Control and Motion Prediction}}
\label{sec.SafeNavigation}

In this section, we demonstrate an application of the adaptive headway control in \refeq{eq.AdaptiveHeadwayControl} and the associated circular and triangular feedback motion predictions in \refeq{eq.CircularMotionPrediction} and \refeq{eq.TriangularMotionPrediction} for safe unicycle path following around obstacles using a time governor \cite{arslan_arXiv2022}.
In brief, a time governor performs an online time parametrization of a reference path for provably correct and safe path following based on the safety assessment of the predicted robot motion \cite{arslan_arXiv2022}, which requires an asymptotically shrinking motion prediction (see \refprop{prop.CircularMotionPredictionRadius} and \refprop{prop.TriangularMotionPredictionRadius}) with Lipschitz-continuous minimum distance to any given (e.g., obstacle) point (see \refprop{prop.CircularMotionPredictionDistance} and \refprop{prop.TriangularMotionPredictionDistance}).

\subsection{Safe Unicycle Path Following via Time Governors}

For ease of exposition, we consider a disk-shaped unicycle robot of body radius $\radius>0$, centered at position $\pos \in \workspace$ with orientation $\ort \in [ -\pi, \pi )$, that operates in a known static compact environment $\workspace \subseteq \R^{2}$ which is cluttered with a collection of obstacles represented by an open set $\obstspace \subset \R^{2}$. 
Hence, the robot's free space, denoted by $\freespace$, of collision-free unicycle positions is given by
\begin{align} \label{eq.FreeSpace}
\freespace \ldf \clist{ \pos \in \workspace \, \big| \,   \ball(\pos,\radius) \subseteq \workspace \setminus \obstspace },
\end{align}
where $\ball(\pos,\radius) \ldf \clist{ \vect{y} \in \R^{\dimspace} \big|  \norm{\vect{y} - \pos } \leq \radius} $ is the Euclidean closed ball centered at $\pos$ with radius $\radius$, and  $\norm{.}$ denotes the standard Euclidean norm.
Suppose  $\path(\pathparam) \!:\! [\smin, \smax] \!\rightarrow\! \freespace$ be a Lipschitz-continuous collision-free reference path that joins  a pair of collision-free start and goal positions $\startpos, \goalpos \!\in\! \freespace$ such that $\path(\smin) \!=\! \startpos, \path(\smax) \!=\! \goalpos$.
Starting at $t = 0$ from the initial path parameter $\pathparam(0) = \smin$, the initial unicycle position $\pos(0) = \startpos$  and some initial unicycle orientation $\ort(0) \in [-\pi, \pi)$, we construct a safe unicycle path following controller with online path time-parametrization, based on the adaptive headway control $\ctrl_{\path(\pathparam)}$ in \refeq{eq.AdaptiveHeadwayControl} towards the path point $\path(\pathparam)$, as
\begin{subequations}\label{eq.SafePathFollowing}
\begin{align}
\dot{\pathparam} &=  \min \plist{ \gain_\safelevel \safedist_\freespace \plist{\motionset_{\ctrl_{\path(\pathparam)\!}}\!(\pos, \ort)} , \!-\gain_\pathparam (\pathparam \!-\! \smax) }\!\!
\\
\dot{\pos} &= \linvel_{\path(s)}(\pos, \ort) 
\\
\dot{\ort} &= \angvel_{\path(s)}(\pos, \ort)
\end{align}
\end{subequations}
where  $\gain_\safelevel, \gain_\pathparam > 0$ are fixed positive control coefficients, and the safety of the unicycle motion is measured by the minimum distance between a (e.g., circular or triangular) feedback motion prediction set $\motionset_{\ctrl_{\path(\pathparam)}}(\pos, \ort)$ of the adaptive headway controller $\ctrl_{\path(\pathparam)}$ and the free space boundary $\partial \freespace$  as
\begin{align}\label{eq.safedist}
\safedist_{\freespace} (\motionset_{\ctrl_{\path(\pathparam)}}\!(\pos, \ort) \!) 
& \!\ldf\! \!\left \{ \begin{array}{@{}l@{\,}l@{}}
\min\limits_{\substack{\vect{a} \in \motionset_{\ctrl_{\path(\pathparam)}}\!\!(\pos, \ort)\\ \vect{b} \in \partial \freespace} }\!\!\!\!\! \!\!\!\!\norm{\vect{a} \!-\! \vect{b}} & \text{, if } \motionset_{\ctrl_{\path(\pathparam)}} \!(\pos, \ort) \!\subseteq\! \freespace \\
0 & \text{, otherwise.}
\end{array}
\right.
\end{align}    
In summary, based on the safety level of the predicted unicycle robot motion, the path parameter $\pathparam$ is continuously increased in \refeq{eq.SafePathFollowing} while the unicycle robot under adaptive headway control $\ctrl_{\path(\pathparam)}$ moves towards the reference path point $\path(\pathparam)$ which acts as a local goal.
If the reference path $\path$ has a nonzero clearance from the free space boundary $\partial \freespace$, an asymptotically shrinking feedback motion prediction with a Lipschitz continuous safety distance \refeq{eq.safedist} ensures that the path parameter trajectory $\pathparam(t)$ and the unicycle position trajectory $\pos(t)$ under the time-governed path following dynamics in \refeq{eq.SafePathFollowing} asymptotically converge to the end of the reference path  with no collision between the robot and obstacles along the way \cite{arslan_arXiv2022}, i.e.,
\begin{align*}
\pos(t) &\in \freespace \quad \forall t \geq 0\\
\lim_{t\rightarrow \infty} \pathparam(t) &= \smax\\
\lim_{t\rightarrow \infty} \pos(t) &= \path(\smax).
\end{align*}

\subsection{Numerical Simulations}
\label{sec.NumericalSimulations}


In this part, we provide numerical simulations\footnote{For all simulations, we set the headway distance coefficient $\hgain \!=\! 0.5$, the control coefficient for the headway reference dynamics $\rgain \!=\! 1$, and the control coefficients for the time governor in \refeq{eq.SafePathFollowing} $\gain_\pathparam \!=\! 4$, $\gain_\safelevel \!=\! 4$. We use the arc-length parametrization of a given reference path $\path(\pathparam)$ such that the reference path length $L$ determines the path parameter range as $[\smin, \smax] = [0, L]$. All simulations are obtained by numerically solving the time-governed unicycle path-following dynamics in \refeq{eq.SafePathFollowing} using the \texttt{ode45} function of MATLAB.} to demonstrate safe unicycle path following based on adaptive headway control and associated circular and triangular feedback motion prediction methods in an office-like environment illustrated in \reffig{fig.SafeUnicycleNavigation}. 
As a baseline ground-truth motion prediction, we use the forward simulation of the adaptive headway motion control assuming the reference path point is kept fixed.
In \reffig{fig.SafeUnicycleNavigation} and \reffig{fig.velocity_profile}, we illustrate the resulting unicycle position trajectories and speed profiles during safe unicycle path following using circular, triangular, and forward-simulation-based motion predictions. 
The resulting unicycle motion significantly differs in terms of the unicycle speed and travel time, see \reffig{fig.velocity_profile}, depending on the accuracy of feedback motion prediction. 
As expected, forward simulation performs the best in terms of average speed and travel time with a significantly higher computational cost because the safety assessment requires the numerical calculation of the unicycle motion trajectory and the computation of the distance-to-collision at each trajectory point.
On the other hand, the triangular unicycle motion prediction shows a comparable performance like forward simulation with a lower computation cost because of the explicit analytical form of the motion prediction set in \refeq{eq.TriangularMotionPrediction} and its simple triangular shape. 
The circular unicycle motion prediction yields the slowest path following motion because it is more conservative and less accurate than the triangular unicycle motion prediction that strongly depends on unicycle position and orientation.
We observe that the fully symmetric circular motion prediction is more cautious about irrelevant sideways collisions with walls while moving along a wall.
Overall, accurate motion prediction is crucial for safe and fast robot motion generation around complex (e.g., dynamic) obstacles.

\begin{figure}[t]
\centering
\begin{tabular}{@{}c@{\hspace{0.005\columnwidth}}c@{\hspace{0.005\columnwidth}}c }
\includegraphics[width = 0.33\columnwidth]{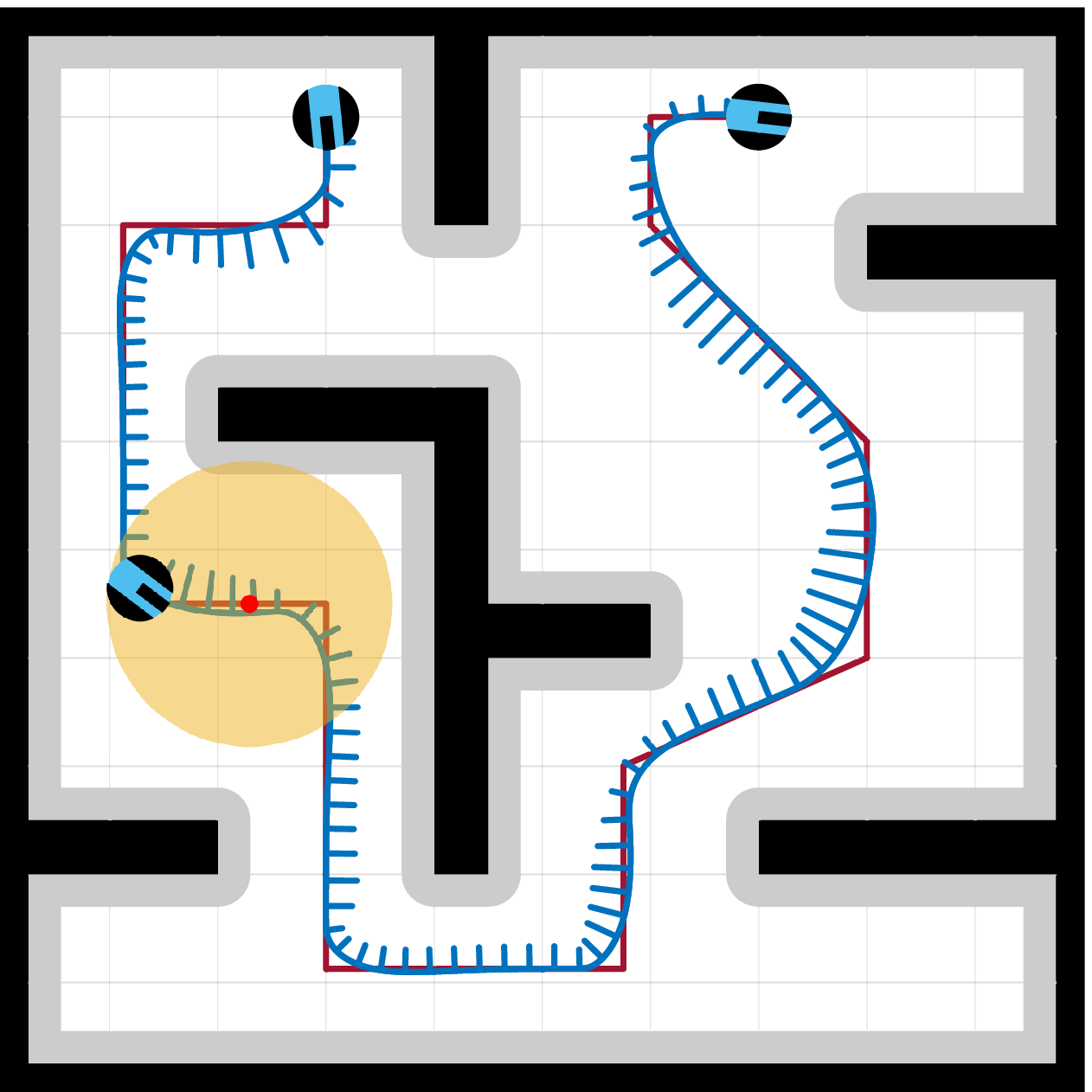} &  
\includegraphics[width = 0.33\columnwidth]{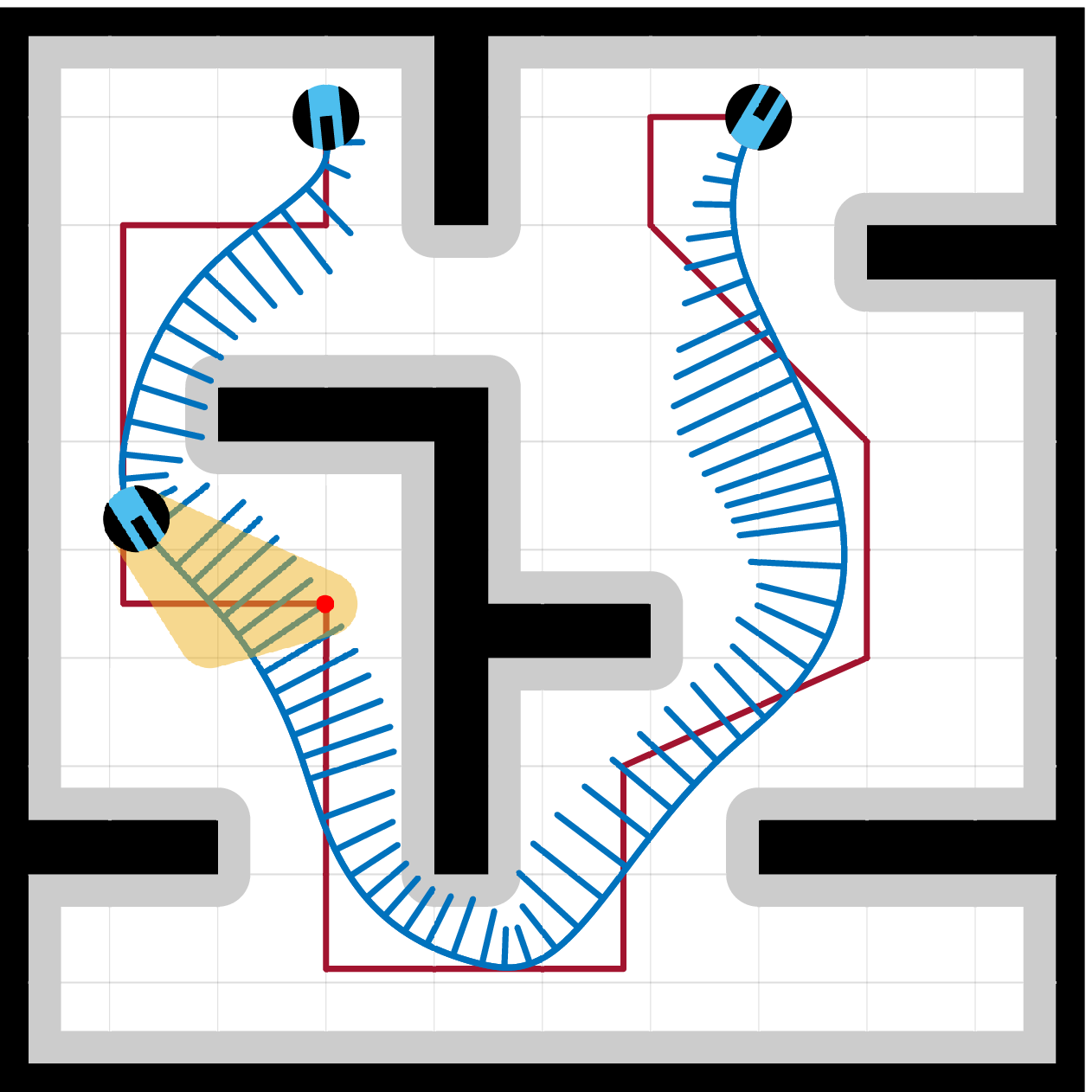} &  
\includegraphics[width = 0.33\columnwidth]{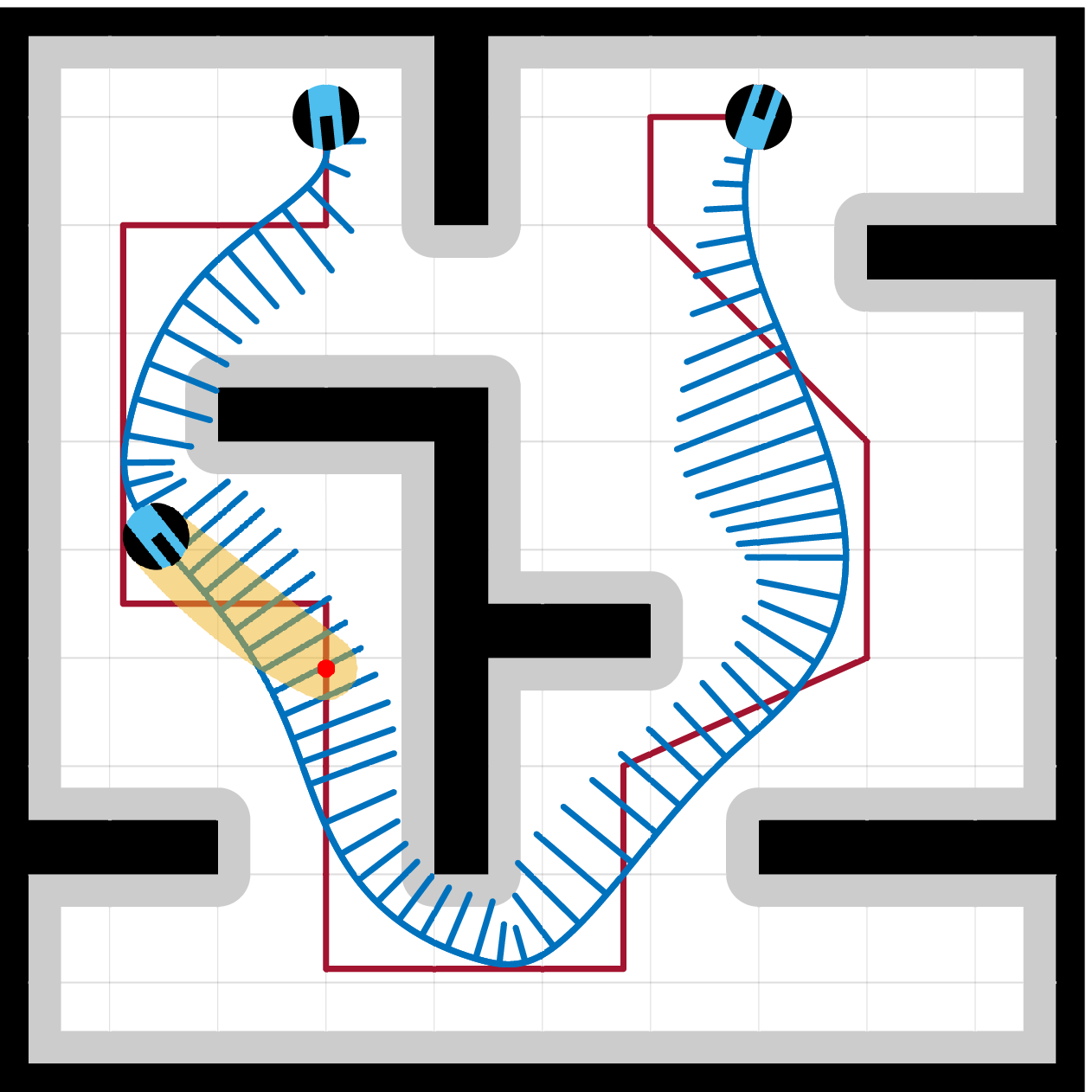} 
\\[-1mm]
\footnotesize{(a)} & \footnotesize{(b)} & \footnotesize{(c)}
\end{tabular}
\vspace{-3mm}
\caption{Time-governed safe unicycle path (red) following in an office-like cluttered environment via adaptive headway control and the presented feedback motion predictions.
The safety of the unicycle motion is constantly verified using (a) circular, (b) triangular, (c) forward-simulation-based motion predictions.
The unicycle robot motion is illustrated by blue lines, where blue bars indicate robot speed.
Yellow regions show an instance of the feedback motion prediction during the robot motion towards the moving reference path point (red point).}
\label{fig.SafeUnicycleNavigation}
\vspace{-3mm}
\end{figure}

\begin{figure}[t]
\centering
\begin{tabular}{@{\hspace{0.02\columnwidth}}c @{\hspace{0.02\columnwidth}} c @{\hspace{0.02\columnwidth}}}
\includegraphics[width = 0.465\columnwidth]{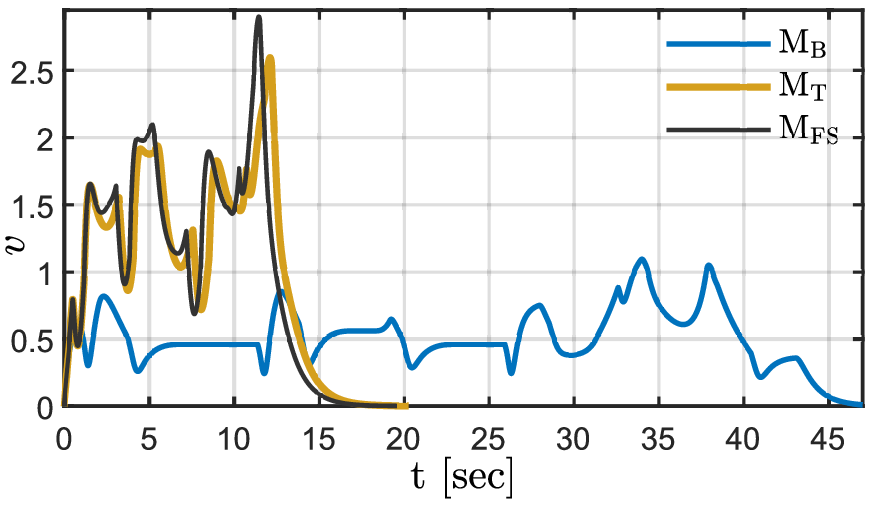} &
\includegraphics[width = 0.465\columnwidth]{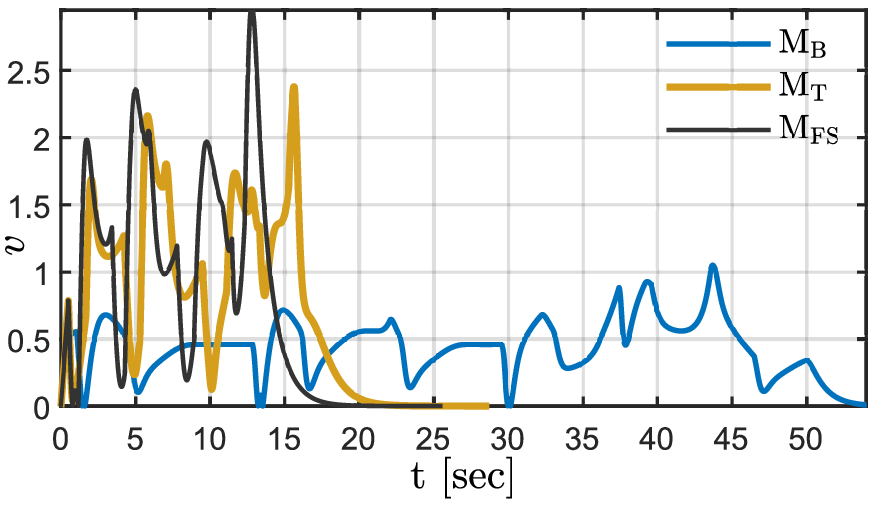} 
\end{tabular}
\caption{Unicycle speed profile during safe path following in an office-like cluttered environment for different
unicycle feedback motion prediction methods: circular $\motionset_\ball$, triangular $\motionset_\tri$, and forward-simulation-based $\motionset_{\mathrm{FS}}$ and different headway distance coefficients $\hgain \!=\! 0.5$ (left), $\hgain\!=\!0.75$ (right). A lower headway distance coefficient results in faster robot motion.}
\label{fig.velocity_profile}
\vspace{-\baselineskip}
\end{figure}

\section{Conclusions}
\label{sec.Conclusions}

In this paper, we design a new unicycle headway controller using an adaptive headway distance that allows the unicycle position to exactly converge a given goal position.
We construct new analytic circular and triangular feedback motion prediction sets that bound the closed-loop unicycle motion trajectory under the adaptive headway controller.
Using online path time parametrization, we present an application of the adaptive headway controller and its feedback motion prediction methods for safe path following of a unicycle robot around obstacles.
In our numerical simulations, we observe that the analytic triangular feedback motion prediction of the adaptive headway controller performs as well as the computationally expensive forward system simulation for capturing the closed-loop unicycle motion accurately and generating safe and fast unicycle motion.

Our current work focuses on  sensor-based safe unicycle motion design using feedback motion prediction in real hardware experiments, especially for safe robot navigation in unknown dynamic environments \cite{arslan_koditschek_IJRR2019}.
We also investigate the use of unicycle feedback motion prediction for multi-robot navigation and crowd simulation \cite{vandenberg_lin_manocha_ICRA2008}.

\bibliographystyle{IEEEtran}
\bibliography{references}

\appendices 

\section{Proofs}

\subsection{Proof of \reflem{lem.SimultaneousConvergence}}
\label{app.SimultaneousConvergence}

\begin{proof}
The sufficiency follows from \refeq{eq.HeadwayPoint} and \refeq{eq.HeadwayDist} as 
\begin{align*}
\pos = \goal \Longrightarrow \hdist = 0 \Longrightarrow \hpos = \pos = \goal.
\end{align*}
The necessity  of the statement can be observed using  the definition of the headway point $\hpos$ in \refeq{eq.HeadwayPoint} as
\begin{align*}
\hpos = \goal & \Rightarrow \hdist \ovect{\ort} \!= \goal \!- \pos \Rightarrow \hdist =  \norm{\goal\! - \pos}   \Rightarrow \pos = \goal
\end{align*}
where the last implication follows from   $\hdist = \hgain \norm{\pos - \goal} $ and $0 < \hgain< 1$.
\end{proof}

\subsection{Proof of \reflem{lem.RobotbtwProjectedandExtended}}
\label{app.RobotbtwProjectedandExtended}

\begin{proof}
If $\pos = \goal$, then all points are located at the goal (i.e., $\hproj=\hext=\goal$) and so the result holds.
Otherwise, to prove that $\pos \in \blist{\hproj, \hext}$, we show below for $\pos \neq \goal$ that
\begin{align*}
\tr{(\pos - \hproj)} \thead = 0,\,\, \tr{(\pos - \hproj)}\nhead  \geq 0, \,\, \tr{(\pos - \hext)}\nhead  \leq 0
\end{align*}
because the tangent $\thead$ of the headway-point motion defines the normal of the line segment between $\hproj$ and  $\hext$, and the normal $\nhead$ is directed from $\hproj$ to $\hext$ (i.e., $\nhead = \frac{\hext - \hproj}{\norm{\hext - \hproj}}$).

$\bullet$ Using \refeq{eq.HeadwayProjectedDefinition}, one can obtain the first condition as 
\begin{align*}
\tr{(\pos - \hproj)} \thead &= \tr{(\pos - \goal - \thead \tr{\thead}(\pos - \goal))} \thead 
\\
&= \tr{(\pos - \goal)} \thead - \tr{(\pos - \goal)} \thead = 0 
\end{align*} 

$\bullet$ The second condition follows from \refeq{eq.HeadwayPoint} and \refeq{eq.HeadwayTangentNormal} as
\begin{align*}
\tr{(\pos - \hproj)}\nhead & = \tr{(\pos - \hpos)} \nhead + \underbrace{\tr{(\hpos - \hproj)} \nhead}_{=0} = - \hdist \ovecTsmall{\ort} \nhead \\
& = \frac{\hdist}{\norm{\goal - \hpos}} \absval{\nvecTsmall{\ort}(\goal - \pos)} \geq 0
\end{align*} 

$\bullet$ Finally, one can verify the third condition  using \refeq{eq.HeadwayTangentNormal} and \reflem{lem.DistanceMetricsOrder} as
\begin{align*}
\tr{(\pos - \hext)}\nhead &=   \tr{(\pos - \hproj)} \nhead + \tr{(\hproj - \hext)} \nhead 
\\
&= \underbrace{- \hdist \ovecTsmall{\ort} \nhead}_{\leq \hdist}  - \underbrace{\frac{\hgain}{\sqrt{1 - \hgain^2}} \norm{\hproj - \goal}}_{\substack{\geq \hgain \norm{\goal -\pos} = \hdist \\ \text{by \reflem{lem.DistanceMetricsOrder}}}}
\leq 0
\end{align*} 
which completes the proof.
\end{proof}

\subsection{Proof of \reflem{lem.DistanceMetricsOrder}}
\label{app.DistanceMetricsOrder}

\begin{proof}
If $\pos = \goal$, then all points are located at the goal (i.e., $\hproj=\hext=\goal$) and so the result holds.
Otherwise, using the definition of the projected robot position in \refeq{eq.HeadwayProjectedDefinition} and the Cauchy-Schwartz inequality, one can obtain the lower bound on $\norm{\pos - \goal}$ for $\pos \neq \goal$ as 
\begin{align*}
\norm{\hproj-\goal} = \absval{\tr{\thead} (\pos - \goal)} \leq \norm{\thead}\norm{\pos - \goal} \leq \norm{\pos - \goal}
\end{align*}
where the last inequality is due to the fact  that $\norm{\thead} \leq 1$. 

By defining $\alpha := \frac{\tr{(\goal - \pos)}}{\norm{\goal- \pos}}\nvect{\ort}$, one can also verify the upper bound on $\norm{\pos - \goal}$ for $\pos \neq \goal$ as\footnote{The relevant terms for the upper bound on $\norm{\pos - \goal}$  for $\pos \neq \goal$ are explicitly given by 
\begin{align*}
&\tr{(\goal- \hpos)}(\goal - \pos) = \norm{\goal - \pos}^2\plist{1 - \hgain  \frac{\tr{(\goal - \pos)}}{\norm{\goal - \pos}}\nvect{\ort}} 
\\
&\norm{\goal - \hpos}^2 
= \norm{\goal - \pos}^2 \plist{1 + \hgain^2 - 2 \hgain  \frac{\tr{(\goal- \pos)}}{\norm{\goal-\pos}}\nvect{\ort}}.  
\end{align*}
}
\begin{align*}
\norm{\hproj-\goal}^2 &= \plist{\tr{\thead}(\goal - \pos)}^2 = \frac{(\tr{(\goal - \hpos)}(\goal - \pos))^2}{\norm{\goal - \hpos}^2}
\\
& = \norm{\pos - \goal}^2 \frac{(1-\hgain \alpha)^2}{1 - 2\alpha \hgain + \hgain^2}
\\
&=\norm{\pos - \goal}^2 \plist{1 - \hgain^2 \frac{1- \alpha^2 }{1 - \alpha^2 + (\hgain - \alpha)^2}}
\\
& \geq \norm{\pos - \goal}^2 \plist{1 - \hgain^2}
\end{align*}
where the inequality is due to $\alpha \in [-1, 1]$.
Hence, the result follows from \refeq{eq.ProjectedExtendedPosition} since  $\norm{\hext - \goal} = \frac{1}{\sqrt{1 - \hgain^2}} \norm{\hproj - \goal}$.
\end{proof}

\subsection{Proof of \reflem{lem.ProjectedExtendedDynamics}}
\label{app.ProjectedExtendedDynamics}

\begin{proof}
If $\pos = \goal$, the result holds because all unicycle positions are the goal ($\pos= \hproj=\hext = \goal$) and the unicycle doesn't move under the adaptive headway controller in \refeq{eq.AdaptiveHeadwayControl}.
 
Otherwise, for $\pos \neq \goal$ the tangent vector $\thead$ of the motion of the headway point $\hpos$ is constant (due to the headway reference dynamics in \refeq{eq.ReferenceDynamics}) and it satisfies  $\thead = \frac{\goal - \hproj}{\norm{\goal - \hproj}}$ (since $\norm{\hproj - \goal} =  \tr{\thead} (\goal - \pos) = \tr{\thead} (\goal - \hproj)$).
Therefore,  we have for $\pos \neq \goal$ that
\begin{align}
\dot{\hproj} &= \thead \tr{\thead} \dot{\pos} = \thead \tr{\thead} \linvel_{\goal}(\pos,\ort) \ovect{\ort}
\\
&= - \kappa(\hproj - \goal)
\end{align}
where 
\begin{align}
\kappa = \linvel_{\goal}(\pos,\ort) \ovecTsmall{\ort} \frac{\goal - \hproj}{\norm{\hproj - \goal}^2}. 
\end{align}
Moreover, one can observe the nonnegativity of $\kappa$  by rewriting it using the definitions of linear velocity in \refeq{eq.AdaptiveHeadwayControlVelocity} and the projected robot position in \refeq{eq.HeadwayProjectedDefinition} as\footnote{Here, we substitute  $\goal - \hproj = \thead \tr{\thead}(\goal - \pos)$ and 
\begin{align*}
\linvel_{\goal}(\pos,\ort)  &= \frac{\rgain \ovecTsmall{\ort} (\goal - \hpos)}{ 1- \hgain \ovecTsmall{\ort} \xvectsmall{\pos} } = \frac{\rgain \norm{\hpos - \goal} \ovecTsmall{\ort} \thead}{ 1- \hgain \ovecTsmall{\ort} \xvectsmall{\pos} }.
\end{align*}
}
\begin{align*}
\kappa &= \frac{\rgain \frac{\norm{\hpos - \goal}}{\norm{\hproj - \goal}} \plist{\ovecTsmall{\ort} \thead}^2 \tr{\thead} (\goal - \pos)}{ 1- \hgain \ovecTsmall{\ort} \xvectsmall{\pos} } 
\\
 &= \frac{\norm{\pos - \goal}^2}{\norm{\hproj - \goal}} \plist{\ovecTsmall{\ort} \thead}^2
\end{align*}
because 
\begin{align*}
\tr{\thead}(\goal - \pos) &=  \frac{1}{\norm{\hpos - \goal}} \tr{(\goal - \hpos)} (\goal - \pos) 
\\
&=  \frac{\norm{\goal - \pos}^2}{\norm{\hpos - \goal}}  \plist{1 - \hgain \ovecTsmall{\ort} \frac{\goal - \pos}{\norm{\goal - \pos}}}.
\end{align*}

Similarly, since both $\thead$ and $\nhead$ are constant under  the adaptive headway controller,  the time rate of change of  the extended unicycle position can be obtained for $\pos \neq \goal$ as 
\begin{align*}
\dot{\hext} &= \dot{\hproj} - \frac{\hgain}{\sqrt{1 - \hgain^2}} \nhead \tr{\thead} \dot{\pos} \\
& = - \kappa (\hproj - \goal) - \frac{\hgain}{\sqrt{1 - \hgain^2}} \nhead \tr{\thead} \linvel_{\goal}(\pos,\ort)  \ovectsmall{\ort}
\\
& =  - \kappa (\hproj - \goal) - \underbrace{\linvel_{\goal}(\pos,\ort)
\frac{\tr{\thead}\!\ovectsmall{\ort}}{\norm{\hproj - \goal}}}_{\kappa} \underbrace{\frac{\hgain}{\sqrt{1 - \hgain^2}} \norm{\hproj - \goal}\nhead}_{= \hext - \hproj}
\\
& = - \kappa (\hext - \goal).
\end{align*}
Hence, the concluding remarks about the decaying distances of the projected and extended unicycle positions and their motion range simply follow from their first-order converging dynamics to the goal position.
\end{proof}

\subsection{Proof of \reflem{lem.GoalAlignment}}
\label{app.GoalAlignment}

\begin{proof}
Under the  adaptive headway controller  in \refeq{eq.AdaptiveHeadwayControl},  the time rate of change of the alignment of the unicycle robot with the goal satisfies
\begin{align*}
    &\!\!\frac{\diff}{\diff t}\!\plist{\! \ovecTsmall{\ort}\! \xvectsmall{\pos} \!} \!=\! \plist{\!\nvecTsmall{\ort}\! \xvectsmall{\pos}\!}^{\!\!2} \plist{\!\frac{\focgain}{\hgain } \!-\! \tfrac{\linvel_{\goal}(\pos, \ort)}{\norm{\goal - \pos}}\!}\\
    &=\!  \plist{\!\nvecTsmall{\ort}\! \xvectsmall{\pos}\!}^{\!\!2} \plist{\!\! \frac{\focgain}{\hgain} \!-\! \frac{ \focgain   \plist{ \ovecTsmall{\ort} \xvectsmall{\pos} \!-\! \hgain \!\!}  }{1 \!-\! \hgain \ovecTsmall{\ort} \xvectsmall{\pos} } \!\!}\\
    &=\! \frac{\rgain}{\hgain}\plist{\!\nvecTsmall{\ort}\! \xvectsmall{\pos}\!}^{\!\!2} \,\, \frac{\!\plist{1 \!-\! 2\hgain\ovecTsmall{\ort} \xvectsmall{\pos} \!+\! \hgain^2 } }{ \plist{1 \!-\! \hgain \ovecTsmall{\ort} \xvectsmall{\pos} }}
\\
&\geq \frac{\rgain}{\hgain}\plist{\!\nvecTsmall{\ort}\! \xvectsmall{\pos}\!}^{\!\!2} \plist{1 \!-\! \hgain \ovecTsmall{\ort} \xvectsmall{\pos} } 
\end{align*}
where the inequality is due to the fact that $0 < \hgain < 1$ and
\begin{align*}
1 \!-\! 2\hgain\ovecTsmall{\ort} \!\!\xvectsmall{\pos} \!+\! \hgain^2 \!>\! \plist{\! 1 \!-\! \hgain\ovecTsmall{\ort} \!\!\xvectsmall{\pos} \!}^{\!2} > 0
\end{align*}
which completes the proof.
\end{proof}

\subsection{Proof of \reflem{lem.EuclideanDistance2Goal}}
\label{app.EuclideanDistance2Goal}

\begin{proof}
Since the adaptive headway controller constantly turns the robot towards the goal, it follows from \reflem{lem.GoalAlignment} that $\ovecTsmall{\ort(t)}\!\! \xvectsmall{\pos(t)} \!>\! \hgain$ for all $t \geq 0$  and $\pos(t) \neq \goal$.
Hence, by construction in \refeq{eq.AdaptiveHeadwayControlVelocity}, having $\ovecTsmall{\ort(t)}\!\! \xvectsmall{\pos(t)} \!>\! \hgain$ ensures  $\linvel_{\goal}(\pos(t), \ort(t)) > 0$ for $\pos(t) \neq \goal$.
Therefore, using these facts, one can verify that the Euclidean distance of the unicycle position to the goal strictly decreases along the unicycle motion trajectory for $\pos(t) \neq \goal $ as 
\begin{align*}
\frac{\diff}{\diff t} \norm{\pos(t)\! -\! \goal}^2 = - 2 \underbrace{\linvel_{\goal}(\pos(t), \ort(t)\!)}_{> 0} \underbrace{\ovecTsmall{\ort(t)}\!\! (\goal\!\!-\!\pos(t))}_{>\hgain } < 0
\end{align*}    
which completes the proof.
\end{proof}

\end{document}